\documentclass[10pt,journal]{IEEEtran}
\usepackage{amsthm}
\usepackage{booktabs}
\usepackage{array}
\usepackage{hyperref}
\usepackage{overpic}
\usepackage{xcolor}

\usepackage{amssymb}
\usepackage{graphicx}
\usepackage{epstopdf}
\usepackage[english]{babel}
\usepackage[numbers]{natbib}
\usepackage{bm}
\usepackage{float}
\usepackage{algorithm}
 \usepackage{dsfont}
\usepackage{multirow}
\usepackage{longtable}
\usepackage{subfigure}
\usepackage{subcaption}
\usepackage{supertabular}
\usepackage{algpseudocode}
\usepackage{amsmath}
\usepackage[justification=centering]{caption}

\usepackage{algorithm,float}
\makeatletter 
\newtheorem{theorem}{Theorem}

\begin{document}

\title{Adversarial Fair Multi-View Clustering}

\author{Mudi Jiang, Jiahui Zhou, Lianyu Hu, Xinying Liu, Zengyou He,  Zhikui Chen
\thanks{M. Jiang, J. Zhou, L. Hu, X. Liu and Z. He are with School of Software, Dalian
University of Technology, Dalian, China.}
\thanks{Z. Chen (corresponding author) is with School of Software, Dalian University of Technology, Dalian,
China, and Key Laboratory for Ubiquitous Network and Service Software
of Liaoning Province, Dalian, China.\protect\\ Email: zkchen@dlut.edu.cn}
\thanks{Manuscript received XXXX 2025; revised XXXX, 2025.}
}
\markboth{}%
{Shell \MakeLowercase{\textit{et al.}}: A Sample Article Using IEEEtran.cls for IEEE Journals}


\maketitle

\begin{abstract}
Cluster analysis is a fundamental problem in data mining and machine learning. In recent years, multi-view clustering has attracted increasing attention due to its ability to integrate complementary information from multiple views. However, existing methods primarily focus on clustering performance, while fairness—a critical concern in human-centered applications—has been largely overlooked. Although recent studies have explored group fairness in multi-view clustering, most methods impose explicit regularization on cluster assignments, relying on the alignment between sensitive attributes and the underlying cluster structure. However, this assumption often fails in practice and can degrade clustering performance. In this paper, we propose an adversarial fair multi-view clustering (AFMVC) framework that integrates fairness learning into the representation learning process. Specifically, our method employs adversarial training to fundamentally remove sensitive attribute information from learned features, ensuring that the resulting cluster assignments are unaffected by it. Furthermore, we theoretically prove that aligning view-specific clustering assignments with a fairness-invariant consensus distribution via KL divergence preserves clustering consistency without significantly compromising fairness, thereby providing additional theoretical guarantees for our framework. Extensive experiments on  data sets with fairness
constraints demonstrate that AFMVC achieves superior fairness and competitive clustering performance compared to existing multi-view clustering and fairness-aware clustering methods.
\end{abstract}


\section{Introduction}
\label{1}
\IEEEPARstart{M}{ulti}-view data \cite{sun2013survey} has been extensively utilized in a wide range of real-world applications in the era of big data, as the diverse information it captures is vital for effective data mining and analysis. For instance, texture and structural features in images, or image frames and audio in movie clips, represent different views of the same entity. Among various data mining tasks, cluster analysis \cite{oyewole2023data} plays a fundamental role in uncovering hidden patterns or structures within data. In this context, multi-view clustering \cite{fang2023comprehensive} has attracted increasing attention due to its ability to integrate heterogeneous yet complementary information from multiple views, enabling more accurate and robust data partitioning compared to single-view approaches.

In recent years, numerous multi-view clustering (MVC) algorithms have been proposed, which can be broadly categorized into graph learning methods, subspace methods, matrix factorization methods, kernel learning methods, and deep learning methods. Although these methods have demonstrated promising performance in terms of accuracy, most of them primarily focus on improving clustering quality, while overlooking fairness concerns that are increasingly critical in sensitive real-world applications. In domains such as recommendation systems, recruitment, and healthcare, unfair clustering results may lead to biased decisions and amplify existing social biases.

Fairness in clustering \cite{chhabra2021overview} is typically defined based on the application context and is generally categorized into two types: individual fairness and group fairness. Individual fairness requires that similar individuals receive similar treatment, which relies on a well-defined similarity metric. Group fairness, on the other hand, is concerned with avoiding systematic disadvantages to any demographic group. Group-level fairness notions are often grounded in the concept of Disparate Impact \cite{rutherglen1987disparate}, which states that no group should be disproportionately harmed or favored by the outcomes of an algorithmic decision-making system. In this work, we focus on group fairness, as it is particularly relevant in real-world scenarios involving population-level equity and social responsibility.

Although fairness-aware learning has been extensively explored in unsupervised learning tasks, research on fairness in multi-view data remains relatively limited. A few recent studies have begun to investigate this problem from different perspectives and have introduced initial attempts to incorporate fairness into multi-view clustering frameworks \cite{zheng2023fairness,zhao2024dfmvc,li2024one}. However, these methods typically impose fairness constraints by directly regularizing the cluster assignment outputs, and often rely on a strong alignment between the distribution of sensitive attributes and the underlying cluster structure. In practice, this reliance may not hold, which can lead to suboptimal clustering results or fairness degradation.

Motivated by the above observation, this paper proposes \textbf{A}dversarial \textbf{F}air \textbf{M}ulti-\textbf{V}iew \textbf{C}lustering (AFMVC), a novel framework that integrates adversarial fairness learning into the  multi-view clustering process. By removing sensitive information prior to clustering, the framework aims to learn fair feature representations that are invariant to protected attributes. Specifically, we first employ autoencoders to extract features from multiple views and then introduce an adversarial objective as a fairness loss to remove sensitive information. A discriminator is trained to predict sensitive attributes, while a gradient reversal layer is applied to the feature space to encourage the encoder to suppress sensitive information. In our framework, a combination of reconstruction, clustering, and fairness losses is optimized to ensure that the learned features are both effective for clustering and invariant to group-specific biases.

To validate the effectiveness of the proposed method, we conduct extensive experiments on data sets with fairness constraints. The results show that our method achieves superior performance in both clustering accuracy and fairness compared to state-of-the-art (SOTA) multi-view and fairness-aware clustering methods.

The main contributions of this paper are as follows:
\begin{itemize}
    \item We propose a novel  multi-view fair clustering framework that integrates adversarial learning to remove sensitive information from the feature space. This approach provides a new perspective on promoting group fairness in the context of multi-view clustering.
    \item A unified training strategy is designed, where the reconstruction loss, clustering loss, and fairness loss are jointly optimized. This design enables the model to learn cluster-discriminative and fair feature representations, without relying on explicit modeling or strict regularization of group-specific output distributions.
    \item Theoretical analysis shows that aligning view-specific clustering results with a fairness-invariant consensus distribution via Kullback-Leibler divergence does not compromise fairness, providing theoretical guarantees for  our framework.
    \item The experimental results demonstrate that our method consistently outperforms  SOTA multi-view clustering and fairness-aware clustering approaches in terms of both clustering accuracy and  fairness.
\end{itemize}

The remainder of this paper is structured as follows. Section \ref{2} reviews the related work. Section \ref{3} introduces the proposed framework in detail. Section \ref{4} reports the experimental results. Finally, Section \ref{5} concludes the paper and outlines potential directions for future research.
\section{Related work}
\label{2}
\subsection{Multi-View Clustering}
As introduced in Section \ref{1}, each category of multi-view clustering methods adopts distinct strategies to address the challenges associated with multi-view data. Graph learning methods \cite{li2021consensus,liang2019consistency} construct a graph for each view to capture similarity structures, and integrate them via fusion or alignment to guide clustering. Subspace methods \cite{li2019flexible,yang2019split} project multi-view data into a shared low-dimensional space that retains common information, enabling consistent and noise-robust clustering. Matrix factorization methods \cite{wang2018multiview,yang2020uniform} decompose each view into low-rank matrices to uncover latent structures, which are then used for clustering. Kernel learning methods \cite{liu2023contrastive,yuan2022robust} map data into high-dimensional spaces using kernel functions to capture nonlinear relations, and perform clustering by combining or aligning multiple kernels. Given the relevance to our method, a brief overview of  deep multi-view clustering approaches is presented in the following subsection, which can be broadly categorized into two types based on whether feature learning and clustering are performed separately or jointly.
\subsubsection{Two-stage deep learning methods}
These methods typically follow a two-stage paradigm, where feature representations are first learned and clustering is subsequently performed in the embedding space. For instance, Li et al. \cite{10.5555/3367243.3367449} proposed DAMC, which employs adversarial training to enhance feature learning. Denoising autoencoders extract latent embeddings, while view-specific discriminators encourage alignment across views by distinguishing real from reconstructed data, spectral clustering is then applied to the shared latent space to obtain the final results. Similarly, Gao et al. \cite{gao2020cross} introduced a framework based on deep convolutional autoencoders, which are trained using a combination of reconstruction loss and a deep Canonical Correlation Analysis (CCA)-based loss function to enhance cross-view correlation. A self-expression layer is incorporated to capture inter-sample relationships, and clustering is subsequently performed on the resulting affinity matrix.

Other works, such as those by Trosten et al. \cite{trosten2021reconsidering} and Qin et al. \cite{qin2021semi}, follow similar designs but incorporate additional mechanisms to enhance representation learning. The former integrates a contrastive module to improve cross-view consistency, while the latter jointly learns shared and view-specific affinity matrices with limited supervision to regularize the clustering structure. 
\subsubsection{One-stage deep learning methods}
In these approaches, feature learning and clustering are simultaneously optimized within a unified deep framework. A class of representative methods in this category is Deep Embedded Clustering (DEC) by Xie et al. \cite{pmlr-v48-xieb16}, where a Kullback-Leibler (KL) divergence-based clustering loss is optimized alongside latent feature representations to encourage cluster-friendly embeddings. Building upon this idea, several extensions have been developed for multi-view scenarios. For example, Xie et al. \cite{8999493} extend DEC by jointly optimizing soft assignments or target distributions across multiple views to enable deep multi-view clustering. Xu et al. \cite{XU2021279} further introduce a collaborative training mechanism in which views iteratively guide each other through a shared auxiliary target distribution. In addition, attention mechanisms \cite{diallo2023auto} and adversarial training \cite{zhou2020end} have also been employed to improve feature fusion and clustering quality within a unified training framework.
\subsection{Fair Clustering}
\subsubsection{Individual fairness}
The notion of individual fairness in clustering was first introduced by Jung et al. \cite{jung_et_al:LIPIcs.FORC.2020.5}, who defined a fairness constraint based on a point-specific radius $r{(x)}$. For a point  $x$ in a data set $X$ of size $n$, let $r{(x)}$ denote the smallest radius such that the closed ball centered at $x$, with radius $r{(x)}$, contains at least $n/k$ points from $X$, given a target of $k$ clusters. This definition captures the intuition that, assuming a uniform random selection of $k$ centers, each point expects to be close to at least one of them. A clustering is considered $\alpha$-fair if every point is assigned to a center within distance $\alpha \cdot r{(x)}$. Building on this definition, Mahabadi et al. \cite{pmlr-v119-mahabadi20a} proposed the $(\beta,\gamma)$-bicriteria approximation framework, where $\beta$ quantifies the approximation ratio with respect to the optimal $\alpha$-fair cost, and $\gamma$ introduces a relaxation of the fairness constraint. Specifically, a solution is said to be a $(\beta,\gamma)$-approximation if its cost is at most $\beta$ times that of the optimal $\alpha$-fair clustering, and every point is assigned to a center within distance $\gamma \cdot \alpha \cdot r{(x)}$.

Subsequent work has focused on improving these approximation guarantees \cite{10.5555/3540261.3541283,pmlr-v151-vakilian22a}, primarily by tightening the bicriteria bounds for clustering objectives based on different $\ell_p$  norms, such as the $\ell_2$-norm in $k$-means and the $\ell_1$-norm in $k$-median. In addition, some efforts have aimed to reduce the computational complexity of fairness-aware clustering algorithms to improve scalability on large-scale data sets \cite{pmlr-v151-chhaya22a, pmlr-v238-bateni24a}.
\subsubsection{Group fairness}
Unlike individual fairness, group fairness is defined with respect to protected attributes such as race and gender. Depending on when fairness constraints are introduced, clustering algorithms for group fairness can be broadly classified into pre-processing, in-processing, and post-processing approaches. 

Pre-processing methods introduce fairness constraints before clustering is performed, typically by transforming the input data so that standard clustering algorithms applied afterward produce fair results. A representative example is fairlet decomposition \cite{NIPS2017_978fce5b,NEURIPS2020_f10f2da9}, which constructs small fair groups (fairlets) prior to clustering. Other related approaches, such as  fair coresets \cite{schmidt2020fair,NEURIPS2019_810dfbbe}, summarize the data to preserve fairness in a scalable manner, while antidote data \cite{pmlr-v171-chhabra22a} aims to improve fairness by augmenting the data set with additional samples. In-processing methods enforce fairness by modifying the clustering algorithm or objective function itself, typically through joint optimization of clustering quality and fairness. Some approaches integrate fairness constraints directly into traditional objectives such as spectral or $k$-median clustering \cite{kleindessner2019guarantees,10.1007/978-3-030-86520-7_47}, while others incorporate fairness through techniques like adversarial learning or multi-objective optimization in deep clustering frameworks \cite{zhang2021deep,li2020deep}. Post-processing methods achieve fairness by modifying clustering results after applying a standard algorithm. Unlike in-processing approaches, they do not change the clustering objective but adjust cluster assignments or centers to meet fairness criteria \cite{kleindessner2019fair,jones2020fair}.

In recent years, group fairness in multi-view clustering has received growing attention, with most existing methods focusing on in-processing strategies. Zheng et al. proposed Fair-MVC \cite{zheng2023fairness}, a fairness-aware multi-view clustering method that explicitly enforces group fairness constraints during clustering. The method maximizes the agreement of soft cluster assignments across views, while forcing the proportion of protected groups in each cluster to closely match that in the overall data set. Zhao et al. \cite{zhao2024dfmvc} proposed a method that incorporates contrastive constraints to learn consistent and discriminative representations across views. To promote fairness, the method guides the clustering assignments of each sensitive subgroup toward a predefined target distribution, preventing any group from being overly concentrated in specific clusters. Li et al. proposed FMSC \cite{li2024one}, a fairness-aware multi-view spectral clustering method that introduces group fairness through an explicit graph-theoretic regularization. Specifically, the method minimizes the average degree of protected group subgraphs within each cluster, which is theoretically shown to align with a classical definition of group fairness.

Despite their different designs, these methods share several limitations. They commonly rely on  assumptions about the alignment between sensitive attribute distributions and the underlying cluster structure, creating a strong dependency between fairness constraints and clustering outcomes. When this alignment fails, the imposed constraints may distort cluster boundaries and degrade performance.  Moreover, fairness is typically introduced through regularization of the clustering assignments, rather than being embedded into the feature learning process, leading to latent representations that retain sensitive information and undermine fairness. In contrast, inspired by the idea of adversarial fairness learning \cite{li2020deep}, our method adopts an adversarial strategy to eliminate sensitive information at the representation level. This enables the model to achieve fairness implicitly, without relying on group-specific modeling or assignment-level fairness constraints.

\section{Method}
\label{3}
In this section, we present the overall structure of the proposed multi-view fair clustering framework, AFMVC, as illustrated in Fig. \ref{fig1}. The framework consists of three main components: Multi-View Feature Reconstruction, Consensus-Guided Clustering, and Adversarial Fairness Learning. The multi-view feature reconstruction module learns compact and informative representations by reconstructing each view through a view-specific autoencoder. The consensus-guided clustering module leverages a fused assignment as a supervisory signal to guide view-specific cluster predictions, encouraging each view to learn clustering structures aligned with the consensus assignment. The adversarial fairness learning module integrates a multi-layer perceptron (MLP)-based discriminator and a gradient reversal layer (GRL) to eliminate sensitive attribute information from the learned representations. During training, the discriminator attempts to predict sensitive attributes from the fused features, while the encoder is trained adversarially—via the GRL—to make such predictions difficult. This adversarial interplay promotes the learning of fair, group-invariant representations at the feature level.

\begin{figure*}[htbp]
	\includegraphics[scale=0.7]{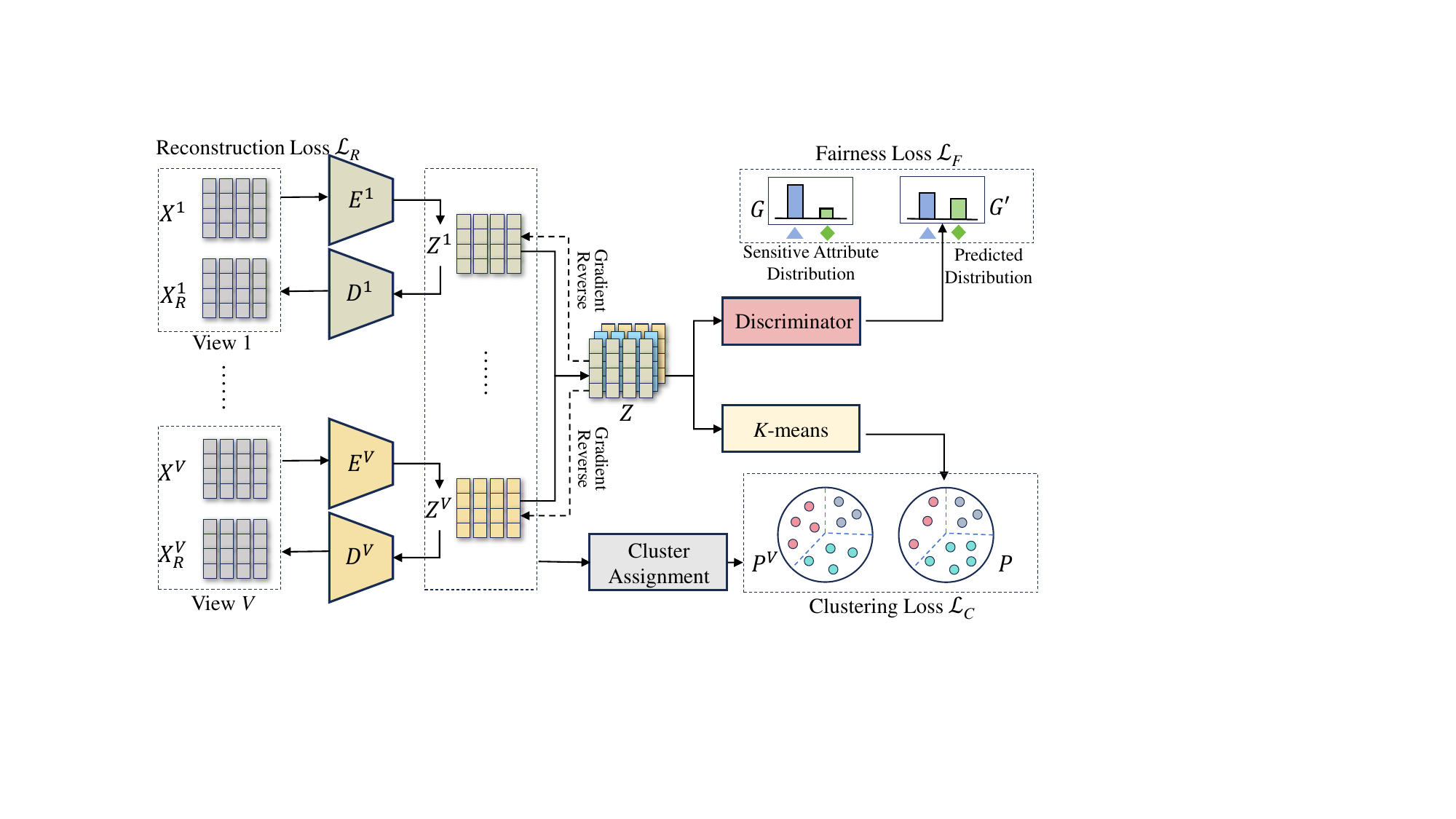}  
	\centering
	\caption{Overview of the proposed adversarial multi-view fair clustering framework.}  
	\label{fig1}   
\end{figure*}
\subsection{Notations and Problem Definition}
The  multi-view fair clustering task aims to partition a set of $N$ instances into $K$ distinct clusters based on information from multiple views.  Each instance is represented as $\{X=X_i^v\}_{v=1}^V$, where $X_i^v$ denotes the representation of the $i$-th instance in the $v$-th view, and is associated with a sensitive attribute $G_i \in \mathcal{G}$. For the $v$-th view, the input data can be organized as a matrix $X^v \in \mathbb{R}^{N \times d_v}$, where 
$N$ is the number of instances and $d_v$ denotes the feature dimension of view $v$. The task requires not only discovering a consistent underlying cluster structure across multiple views to achieve high-quality clustering results, but also ensuring that the outcome satisfies group fairness constraints with respect to the sensitive attribute.
\subsection{Multi-View Feature Reconstruction}
To capture view-specific representations while preserving the essential information of each view, we adopt a set of view-specific autoencoders to perform feature reconstruction for each view independently.  For the $v$-th view, the encoder $E^v$, parameterized by $\theta^v$, maps an input sample $X_i^v$ into a latent feature vector
\begin{equation}
\label{z^v}
   Z_i^v=E^v(X_i^v;\theta^v).
\end{equation}
The decoder $D^v$,  parameterized by $\phi^v$, then reconstructs the  input as:
\begin{equation}
    X_{i,R}^v = D^v(Z_i^v;  \phi^v).
\end{equation}
To preserve essential structural information, the reconstruction loss for view $v$ is defined as
\begin{equation}
\mathcal{L}_{R}^v = \sum_{i=1}^{N}\| X_{i,R}^v - X_i^v \|_2^2.
\end{equation}
The overall reconstruction objective across all views is then given by:
\begin{equation}
\label{loss_r}
\mathcal{L}_{R} = \sum_{v=1}^V \sum_{i=1}^{N}\| X_{i,R}^v - X_i^v \|_2^2.
\end{equation}
This loss encourages each autoencoder to retain view-specific information that is essential for downstream clustering and fairness modeling. In addition, the reconstruction process serves as a pretraining stage that provides well-initialized parameters, facilitating more stable and effective optimization in subsequent components.
\subsection{Consensus-Guided Clustering}
This subsection introduces the consensus-guided clustering module, which functions as the fundamental model for achieving clustering consistency. Its main objective is to align the clustering structures of all views by introducing a shared consensus distribution as a supervision signal.

After obtaining the view-specific latent representations $\{Z^v\}^V_{v=1}$ from the reconstruction module, we compute the fused representation $Z \in \mathbb{R}^{N \times d}$ by concatenating the latent features of all views along the feature dimension:
\begin{equation}
\label{z}
    Z_i=[Z_i^1;Z_i^2;...;Z_i^V],
\end{equation}
where $d=\sum_{v=1}^V d_v$ is the total feature dimension across all views. We then apply $k$-means clustering to $Z$ to generate the initial cluster assignments. These cluster labels are subsequently converted into one-hot vectors, forming the consensus target distribution $P \in \mathbb{R}^{N \times K}$, where $P_{ij}$ denotes the probability of assigning instance $i$ to cluster $j$ (i.e., $P_{ij}=1$ if instance $i$ is assigned to cluster $j$, and 0 otherwise). This distribution acts as a unified objective that guides all views toward a common clustering structure.

For each view $v$, we compute the soft cluster assignment matrix $Q^v\in \mathbb{R}^{N \times K}$  by measuring the similarity between each latent  representation $Z_i^v$ and a set of view-specific cluster centroids 
$\{\mu_j^v\}^K_{j=1}$ using a Student's t-distribution kernel:
\begin{equation}
\label{Q}
Q_{ij}^v = \frac{(1 + \| Z_i^v - \mu_j^v \|^2 / \alpha)^{-\frac{\alpha+1}{2}}}{\sum_{j'=1}^K (1 + \| Z_i^v - \mu_{j'}^v \|^2 / \alpha)^{-\frac{\alpha+1}{2}}},
\end{equation}
where $Q^v_{ij}$ represents the predicted probability that instance $i$ belongs to cluster $j$ in view $v$, and $\alpha$ represents the degree of freedom of  the Student's t-distribution, set to 1 by default in our experiments.
The cluster centroids $\{\mu_j^v\}$ are randomly initialized and updated jointly with the encoder parameters via backpropagation.

To align each view's clustering structure with the consensus distribution, we minimize the  KL divergence between the one-hot consensus target $P$ and each view-specific soft assignment $Q^v$:
\begin{equation}
\label{L_C}
    \mathcal{L}_{C} = \sum_{v=1}^V \sum_{i=1}^{N}\sum_{j=1}^{K}P_{ij} \log \frac{P_{ij}}{Q_{ij}}.
\end{equation}
In practice, we do not update the consensus target $P$  at every training iteration. Instead, $k$-means clustering is recomputed every $T$ epochs (e.g., every 50 epochs). This periodic update strategy strikes a balance between stability and adaptability: updating too frequently could lead to optimization oscillations before the encoders sufficiently adapt to previous targets, while updating too infrequently may cause the model to rely on outdated clustering assignments.

By periodically refreshing the consensus targets and aligning view-specific predictions with them, our framework encourages all views to produce clustering assignments that are mutually consistent, thereby enhancing the overall coherence and quality of the multi-view clustering results.
\subsection{Adversarial Fairness Learning}
To further promote fairness in multi-view clustering, we integrate an adversarial training strategy into the proposed framework. The key idea is to remove sensitive attribute information from the fused latent representations, thereby reducing the influence of protected attributes on the clustering results.

In this adversarial setup, the encoder network functions as a generator that aims to produce latent features free of sensitive information, while an adversarial network $D$ acts as a discriminator attempting to predict the sensitive attribute $G$ from the fused representations $Z$. To facilitate this adversarial interaction, we employ a gradient reversal layer (GRL) between the encoder and the discriminator. During the forward pass, the GRL remains inactive and allows the encoded features to flow into the discriminator unchanged. In the backward pass, it multiplies the gradients by a negative scalar, forcing the encoder to update in the opposite direction of the discriminator’s objective, thereby removing sensitive information from the learned representations.

Formally, given $Z \in \mathbb{R}^{N \times d}$, the adversarial network, parameterized by $\omega$,  predicts the sensitive attribute probabilities $G' \in \mathbb{R}^{N \times |G|}$ as:
\begin{equation}
\label{G'}
    G'= D(Z;\omega),
\end{equation}
where $|G|$ denotes the number of sensitive attribute classes.

Moreover, to stabilize the adversarial training process, the coefficient of the gradient reversal layer is dynamically adjusted during training according to a sigmoid-based schedule:
\begin{equation}
\label{coeff}
\text{coeff} = \frac{2}{1 + e^{-\beta \times iter / n}} - 1,
\end{equation}
where $\beta$ controls the growth rate, $iter$ denotes the current training iteration, and $n$ is the total number of training iterations. This progressive adjustment allows the model to initially focus on clustering performance and gradually enforce stronger fairness constraints as training progresses.

The fairness loss is then defined as the standard multi-class cross-entropy loss:
\begin{equation}
\label{L_F}
\mathcal{L}_{F} = -\frac{1}{N} \sum_{i=1}^{N} \sum_{c=1}^{|G|} \mathds{1}(G_i = c) \log G'_{i,c},
\end{equation}
where $\mathds{1}(G_i=c)$ is an indicator function that equals 1 if the true sensitive label of instance $i$ is $c$, and 0 otherwise,  $G'_{i,c}$ denotes the predicted probability that instance $i$ belongs to sensitive attribute class $c$.

By optimizing this adversarial objective, the encoder is encouraged to learn latent features that are discriminative for clustering while being invariant to sensitive attributes.
\subsection{Overall Loss Function}
The overall training objective of our proposed framework integrates three key components: reconstruction loss, clustering loss, and fairness loss. Each component plays a distinct role in ensuring that the learned representations are informative, clusterable, and group-invariant.

Formally, the training process can be formulated as the following minimax saddle point optimization:

\begin{equation}
\max_{\theta, \phi, \mu} \quad \lambda_F\mathcal{L}_{F} - \mathcal{L}_{R} - \lambda_C \mathcal{L}_{C},
\end{equation}

\begin{equation}
\min_{\omega} \quad \lambda_F\mathcal{L}_{F}.
\end{equation}
In this formulation, the encoders are trained to minimize the reconstruction loss $\mathcal{L}_{R}$ and the clustering  loss $\mathcal{L}_{C}$ while maximizing the fairness loss $\mathcal{L}_{F}$ via a gradient reversal mechanism. Simultaneously, the adversarial network seeks to minimize the fairness loss $\mathcal{L}_{F}$ to accurately predict the sensitive attributes. The balance between the clustering  term and fairness term is controlled by the hyper-parameters $\lambda_F$ and $\lambda_C$. 
The complete training procedure is outlined in Algorithm~\ref{algorithm1}.

\begin{algorithm}[htb]
\caption{Training Procedure for Adversarial Fair Multi-View Clustering (AFMVC)}
\label{algorithm1}
\begin{algorithmic}[1]
\Require Multi-view data set $X$; Sensitive attributes $G$; Number of clusters $K$; Trade-off hyper-parameters $\lambda_{C}, \lambda_{F}$; Pseudo-label update interval $T$; Number of iterations $n$; Growth rate $\beta$.
\Ensure Final cluster assignments $\{y_i\}_{i=1}^N$.

\State Pre-train view-specific autoencoders by minimizing $\mathcal{L}_R$ according to Eq. \ref{loss_r}.
\State Obtain initial latent representations $Z$ by encoding $X$ with pre-trained autoencoders.
\State Initialize consensus target distribution $P$ by applying $k$-means on $Z$.
\For{ $i = 1$ to $n$}
    
    \For{each mini-batch}
        \State Compute latent representations $Z^v$, $Z$ according to Eqs. \ref{z^v} and \ref{z}, respectively.
        \State Compute soft cluster assignments $Q^v$ for each view according to Eq. \ref{Q}.
        \State Compute sensitive attribute probabilities for each instance according to Eq. \ref{G'}.
        \State Compute  $\mathcal{L}_{R}$, $\mathcal{L}_{C}$, and  $\mathcal{L}_{F}$ according to Eqs. \ref{loss_r}, \ref{L_C} and \ref{L_F}, respectively.
        \State Update model parameters $\theta$ (encoder), $\phi$ (decoder), $\mu$ (cluster centroids), and $\omega$ (adversarial network) using Adam optimizer, incorporating the gradient reversal mechanism.
        \State Gradually increase the gradient reversal coefficient according to Eq. \ref{coeff}.
    \EndFor
    \If{$i \; \% \; T = 0$}
        \State Update consensus target distribution $P$ by running $k$-means on $Z$.
    \EndIf
\EndFor
\State Obtain final cluster assignments $\{y_i\}_{i=1}^N$ from $Z$ via $k$-means.
\end{algorithmic}
\end{algorithm}

\subsection{Time Complexity Analysis}
We analyze the time complexity of the proposed AFMVC framework based on the primary computational steps during training. Let $L$ denote the number of layers in each encoder, decoder, and discriminator, and let $M$ be the maximum hidden layer size. In each training iteration, $V$ view-specific encoders and decoders, along with a discriminator, are involved in both forward and backward propagation. These operations require $O(VLM^2)$ time per instance. Additionally, computing soft cluster assignments for each view—by calculating distances to $K$ cluster centroids—incurs an extra $O(Kd_v)$ cost per instance. Since the above processes dominate the training cost, the overall time complexity of AFMVC across $n$ epochs is $O(N n  (VL M^2 + K d))$.

\subsection{Theoretical Analysis}
In this section, we theoretically demonstrate that fairness in the consensus clustering result can propagate to individual views through KL alignment. Specifically, by minimizing the KL divergence between each view-specific clustering and the distribution derived from the fair consensus representation, the model aligns predictions across views while preserving the encoder's ability to generate fair representations. This indicates that KL alignment is fairness-preserving, and serves as a complementary mechanism to our adversarial fair multi-view clustering.

Define $P$ as the clustering assignment derived from the fused representation $Z$, and denote $Q^{(v)}$ as the clustering assignment obtained from the representation of the $v$-th view. Let $\mathbb{P}$ and $\mathbb{Q}^{(v)}$ represent the joint distributions over $(P, G)$ and $(Q^{(v)}, G)$ respectively. we then have the following theorem:

\begin{theorem}
If  $P$ is independent of the sensitive attribute $G$ and $D_{\mathrm{KL}}(\mathbb{Q}^{(v)}\|\mathbb{P}) \leq \varepsilon$, then the mutual information between  $Q^{(v)}$ and  $G$ is upper-bounded by:
\begin{equation}
    I(Q^{(v)}; G) \leq \frac{1}{2} \sqrt{ \frac{\varepsilon}{8} } \cdot \ln \frac{2}{\varepsilon} + \mathcal{O}( \sqrt{\varepsilon} ).
\end{equation}
\end{theorem}

\begin{proof}
Let $\mathcal{P}$ denote the set of possible cluster labels (e.g., $\{1,...,K\}$), and $\mathcal{G}$ the set of sensitive attribute categories. Since $P$ is independent of $G$, then for all $p \in \mathcal{P}$, $g \in \mathcal{G}$, we have
\begin{equation}
\textstyle \Pr_{\mathbb{P}}(P = p, G = g) = \textstyle \Pr_{\mathbb{P}}(P = p) \cdot \textstyle \Pr_{\mathbb{P}}(G = g).
\end{equation}
Hence, we have 
\begin{align}
    &I(P; G)  \nonumber\\
    =&\sum_{(p,g) \in \mathcal{P} \times \mathcal{G}} \textstyle \Pr_{\mathbb{P}}(P = p, G = g) \cdot \log \left( \frac{\Pr_{\mathbb{P}}(P = p, G = g)}{\Pr_{\mathbb{P}}(P = p)\Pr_{\mathbb{P}}(G = g)} \right) \nonumber\\
    =& 0.
\end{align}
According to Pinsker’s inequality \cite{csiszar2011information}, since $D_{\mathrm{KL}}(\mathbb{Q}^{(v)} \| \mathbb{P}) \leq \varepsilon$, the total variation (TV) distance between $\mathbb{Q}^{(v)}$ and $\mathbb{P}$ is bounded by:
\begin{align}  
&\mathrm{TV}(\mathbb{Q}^{(v)}, \mathbb{P}) \nonumber\\
= &\frac {\sum_{(p,g) \in \mathcal{P} \times \mathcal{G}}  \left| \textstyle \Pr_{\mathbb{Q}^{(v)}}(Q^{(v)}=p, G=g) - \textstyle\Pr_{\mathbb{P}}(P=p, G=g) \right|} {2}\nonumber\\
\leq &\sqrt{ \frac{1}{2} \varepsilon } := \eta.
\end{align}
To obtain a sharper bound, we further invoke the continuity result of mutual information from \cite{prelov2008mutual}. Specifically, they show that for two joint distributions defined over finite discrete variables, if their total variation distance satisfies $\mathrm{TV}(\mathbb{Q}^{(v)}, \mathbb{P}) \leq \eta$, then the mutual information difference satisfies:
\begin{equation}
    |I(Q^{(v)}; G) - I(P; G)| \leq \frac{\eta}{2} \ln \frac{1}{\eta} + \mathcal{O}(\eta).
\end{equation}
Since $I(P; G) = 0$ by the independence assumption, we obtain:
\begin{equation}
    I(Q^{(v)}; G) \leq \frac{\eta}{2} \ln \frac{1}{\eta} + \mathcal{O}(\eta).
\end{equation}
Substituting $\eta = \sqrt{\frac{1}{2} \varepsilon}$ into the inequality yields:
\begin{align}
\label{final}
    I(Q^{(v)}; G) 
    &\leq \frac{1}{2} \cdot \sqrt{\frac{1}{2} \varepsilon} \cdot \ln \left( \frac{1}{\sqrt{\frac{1}{2} \varepsilon}} \right) + \mathcal{O}\left( \sqrt{\frac{1}{2} \varepsilon} \right) \nonumber\\
    &= \sqrt{ \frac{\varepsilon}{8} } \cdot \ln \left( \sqrt{ \frac{2}{\varepsilon} } \right) + \mathcal{O}( \sqrt{\varepsilon} ) \nonumber\\
    &= \sqrt{ \frac{\varepsilon}{8} } \cdot \left( \frac{1}{2} \ln \frac{2}{\varepsilon} \right) + \mathcal{O}( \sqrt{\varepsilon} ) \nonumber\\
    &= \frac{1}{2} \sqrt{ \frac{\varepsilon}{8} } \cdot \ln \frac{2}{\varepsilon} + \mathcal{O}( \sqrt{\varepsilon} ).
\end{align}
For example, when $D_{\mathrm{KL}}(\mathbb{Q}^{(v)} | \mathbb{P}) \leq 0.1$ (i.e., $\varepsilon=0.1$), substituting this into the upper bound in Equation \ref{final} yields $I(Q^{(v)}; G) \lesssim 0.167$, ignoring higher-order terms. 

This completes the proof that, under KL alignment with a fairness-invariant consensus clustering result, the mutual information between view-specific clustering assignments and the sensitive attribute remains provably bounded. This finding highlights that the use of KL alignment, while primarily intended to improve cross-view consistency, does not undermine fairness—a property particularly important when combined with fairness-enhancing strategies.
\end{proof}

\section{Experiments}
\label{4}
In this section, we conduct experiments to evaluate the effectiveness of the proposed method. All experiments were performed on a PC equipped with an Intel Core i7-10700F CPU (2.90 GHz), 16 GB of RAM, and an NVIDIA GeForce RTX 1660 GPU with 6 GB of memory.
\begin{table*}[htbp]
	\caption{Comparison of various clustering methods in terms of two accuracy metrics and one fairness metric. For each metric across the data sets, the highest score is highlighted in red, and the second-highest score is highlighted in blue.}
	\label{compare}
	\centering

		\begin{tabular}{c|cccccccccc}
			
			\toprule
			Data set &Metrics&BFKM&VFC&FFC&CHOC&MCPL&CGL&FairMVC&FMSC&AFMVC\\
			\midrule
			&   ACC    &0.370 &0.370   &0.368 &0.373 &0.330 &0.368  &\textcolor{blue}{0.402} &0.387 &\textcolor{red}{0.465} \\
			Credit Card &    NMI    & \textcolor{blue}{0.188}&0.186   &0.187 &0.172 &0.104 &0.123  &0.117 &0.126 &\textcolor{red}{0.210} \\
			&   BAL   &0.348 &0.341   &0.340 &0.350 &\textcolor{blue}{0.356} &0.355  &0.341 &\textcolor{red}{0.357} &0.352 \\
			\midrule
                &   ACC    &0.635 &0.633   &0.633 &\textcolor{blue}{0.707} &0.672 &0.641  &0.623 &0.652 &\textcolor{red}{0.749} \\
			Bank Marketing &    NMI    &0.056 &0.056   &0.056 &0.064 &\textcolor{blue}{0.068} &0.061  &0.033 &0.045 &\textcolor{red}{0.094} \\
			&    BAL   &0.289 &0.289   &0.286 &\textcolor{blue}{0.302} &0.288 &0.285  &0.288 &\textcolor{red}{0.305} &\textcolor{blue}{0.302} \\
                \midrule
                &   ACC    &0.550 &0.542   &0.549 &0.520 &0.839 &\textcolor{red}{0.888}  &0.588 &\textcolor{blue}{0.862} &0.836\\
			Law School &    NMI    &\textcolor{red}{0.081} &0.063   &\textcolor{red}{0.081} &0.072&0.050 &0.078  &0.077 &0.060 &0.047\\
			&    BAL   &\textcolor{blue}{0.430}&0.428   &0.428 &0.414 &0.417 &0.362  &0.424 &\textcolor{red}{0.432}&0.420 \\
			\midrule
			&   ACC    &0.797 &0.925   &- &\textcolor{red}{1} &0.831 &\textcolor{blue}{0.997}  &- &0.293 &0.864 \\
			Mfeat &    NMI    &0.754 &0.856   &- &\textcolor{red}{1} &0.814 &\textcolor{blue}{0.994}  &- &0.211 &0.825 \\
			&    BAL    &0.432 &0.434   &- &\textcolor{red}{0.455} &0.404 &\textcolor{blue}{0.450}  &- &0.412 &0.440 \\
			\midrule
			&   ACC    &0.689 &0.750  &- &\textcolor{red}{0.818} &0.699 &0.599  &- &\textcolor{blue}{0.806} &0.750 \\
			COIL &    NMI    &0.784 &0.821   &- &\textcolor{red}{0.906} &0.831 &0.818  &- &\textcolor{blue}{0.895} &0.840 \\
			&    BAL    &\textcolor{red}{0.371} &0.347   &- &0.347 &0.347 &0.347  &- &0.347 &\textcolor{blue}{0.365} \\
            \midrule
			&   ACC    &0.608 &0.644   &- &0.684 &0.674 &\textcolor{blue}{0.699}  &- &0.600 &\textcolor{red}{0.733} \\
			Mean Value &    NMI    &0.373 &0.396   &- &\textcolor{red}{0.443}&0.373 &\textcolor{blue}{0.415}  &- &0.267 &0.403 \\
			&   BAL    &\textcolor{blue}{0.374} &0.369   &- &\textcolor{blue}{0.374} &0.362 &0.360  &- &0.371 &\textcolor{red}{0.376} \\
            \midrule
			&   ACC    &5.9 &5.5   &- &\textcolor{blue}{3.4} &5.0 &4.5  &- &3.6 &\textcolor{red}{2.7} \\
			Mean Rank &    NMI    &4.5 &4.8   &- &\textcolor{red}{3.0}&5.6 &4.4  &- &6.0 &\textcolor{blue}{3.6} \\
			&    BAL    &3.7 &4.9   &- &4.3 &5.5 &5.6  &- &\textcolor{red}{2.8} &\textcolor{blue}{3.5} \\
			\bottomrule
	\end{tabular}
\end{table*}
\subsection{Experiment Setting}
\subsubsection{Data Sets}
\begin{table}[htbp]
	\caption{The summary statistics on data sets used in the performance evaluation.}
	\label{dataset}
	\centering
        \setlength{\tabcolsep}{0.62mm}{
	\begin{tabular}{ccccc}
		\toprule
		Data set &  \#Samples &  \#Clusters &\#Features & Sensitive Feature \\
		\midrule
		Credit Card &  5000   &  5   &  22/22   &   Gender      \\
            Bank Marketing &  2907   &  2   &  12/12   &   Marital      \\
            Law School &  10000   &  2       &  10/10   &  Gender   \\
		Mfeat &  2000   &  10   &  216/76/64/6/240/47   &   Synthetic Binary    \\
		COIL &  1440   &  20      &  1021/3304/6750  &  Synthetic Binary  \\
		\bottomrule
	\end{tabular}}
\end{table}
We conduct experiments on five data sets with fairness constraints, including Credit Card, Bank Marketing, Law School, Mfeat and COIL. Specifically, for the single-view data sets—Credit Card, Bank Marketing, and Law School—we follow \cite{zheng2023fairness} and construct two views by applying non-linear transformations (e.g., Sigmoid and ReLU). For natural multi-view data sets Mfeat and COIL, we follow \cite{pmlr-v54-zafar17a} and randomly assign each instance to one of two protected groups using a Bernoulli distribution with $p=0.5$. Details of each data set, including the number of samples, features, clusters, and sensitive attributes, are summarized in Table \ref{dataset}. Note that although the original Credit Card and Law School data sets contain 29537 and 18692 samples respectively, we randomly subsample 5,000 and 10,000 instances from them to ensure that all comparison methods—including FMSC, MCPL, and CHOC—can be executed successfully, as some of them fail to handle large-scale data sets. 
\subsubsection{Comparison Methods}
The following clustering algorithms are used for comparison with AFMVC and can be categorized into three groups.

\textit{Single-view fair clustering methods (all views are concatenated as the input for these methods)}:
\begin{itemize}
    \item BFKM \cite{pan2023balanced}: This method incorporates fairness and balance constraints into the $k$-means objective function by penalizing group representation deviation and cluster size imbalance,  and optimizes the objective via coordinate descent. In our implementation, the fairness parameter $\rho$ and the balance parameter $\lambda$ are  set to 2000 and 10000 for all data sets, respectively.
    \item VFC \cite{ziko2021variational}: This framework introduces a variational formulation that integrates a KL-divergence-based fairness penalty with diverse clustering objectives, enabling flexible trade-off between clustering quality and demographic balance via a unified bound-optimization scheme. The trade-off parameter $\lambda$ is adaptively selected from a predefined range to satisfy a target level of fairness, following the strategy described in the original implementation.
    \item FFC \cite{pan2023fairness}:  A multi-stage method that enforces fairness in initialization, relaxes constraints to improve clustering quality, and refines results via fairness-preserving local search. The balance threshold $\mu$ is set to 0.8.
\end{itemize}
\begin{figure*}[htbp]
	\includegraphics[scale=0.7]{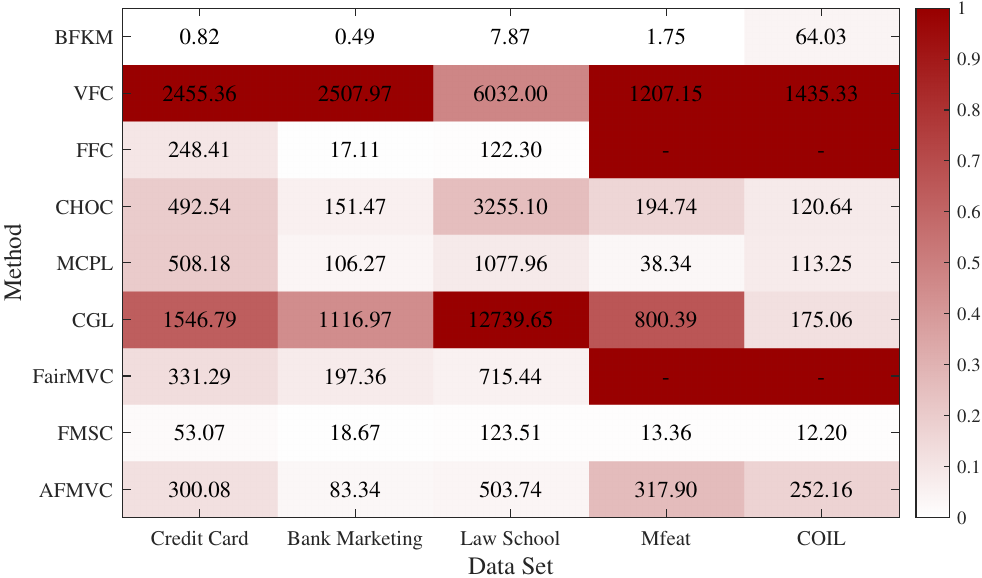}  
	\centering
	\caption{Running time comparison across different data sets. The coloring of the heatmap is determined by the normalized running time: for each data set, the running time is scaled to the range $(0, 1]$ by dividing each value by the maximum running time within that data set. The corresponding color scale, where lighter shades indicate shorter running time, is illustrated in the legend on the right side.}
	\label{time}   
\end{figure*}

\textit{State-of-the-art multi-view clustering methods}:
\begin{itemize}
    \item CHOC \cite{you2024consider}: The model constructs  view-specific graphs to capture consistent and specific structures, integrates them into a comprehensive affinity matrix, and optimizes via the alternating direction method of multipliers. In the experiments, both parameters $\beta$ and $\lambda$ are set to 100.
    \item MCPL \cite{cai2024multi}: This approach combines latent and original data representations by employing pseudo-labels and latent graph recovery, and further improves clustering via a refined label fusion strategy. The parameter settings, following the original paper, are  $\alpha=1000$, $\lambda=1000$, $\beta=0.05$, $\gamma=0.1$ and $\mu=0.0005$.
    \item CGL \cite{li2021consensus}: This method unifies spectral embedding and low-rank tensor learning within a joint framework to learn a consensus graph in the embedded space. In our experiments, the parameters $\lambda$ and $C$ are both set to 1, and the number of nearest neighbors $k$ is set to 15.
\end{itemize}

\textit{Multi-view fair clustering methods}:
\begin{itemize}
    \item FairMVC \cite{zheng2023fairness}: This model embeds group fairness into multi-view clustering by constraining the distribution of sensitive attributes in each cluster to match the global distribution, while enhancing feature representation through contrastive regularization. The parameters $\alpha, \beta, \gamma$ are set to 0.01, 0.1, 0.005, respectively.
    \item FMSC \cite{li2024one}: This approach introduces a one-stage spectral clustering framework that integrates a fairness-aware regularizer, derived from graph theory, to directly yield fair clustering results without requiring post-processing.  The fairness hyper-parameter is set to 0.01.
\end{itemize}
For our method, the hyper-parameters $\lambda_{C}$ and $ \lambda_{F}$ are uniformly set to 0.1 and 0.01, respectively, across all data sets. The update interval for the distribution $P$ is set to 50, and the number of training epochs $n$ is fixed at 1000 and the growth rate $\beta$ is set to 10. All experiments are repeated ten times on each data set, and the average results are reported. 

\subsubsection{Evaluation Metrics}
We evaluate the clustering performance of the proposed method using Clustering Accuracy (ACC) and Normalized Mutual Information (NMI), which are widely used to measure the consistency between predicted cluster assignments and ground-truth labels. For both metrics, higher values represent better clustering performance.

To assess fairness across different algorithms, we employ the Balance (BAL) metric \cite{zheng2023fairness}, which is defined as:
\begin{equation}
\text{BAL} = \min_i \left( \frac{\min_j |\Omega_i \cap G_j|}{|\Omega_i|} \right),
\end{equation}
where \(\Omega_i\) denotes the set of instances assigned to the \(i\)-th cluster, and \(G_j\) represents the set of instances belonging to the \(j\)-th sensitive group. A higher BAL value indicates better group fairness across clusters.

\begin{table*}[htbp]
\caption{Ablation study on different loss combinations. The best result for each metric is marked in red. }
\label{ablation}
\centering
\renewcommand\arraystretch{1.2}
\setlength{\tabcolsep}{1mm}{
\begin{tabular}{c|ccc|ccc|ccc|ccc|ccc|ccc}
\toprule
\multirow{2}{*}{} & \multicolumn{3}{c|}{Loss Components} & \multicolumn{3}{c|}{Credit Card} & \multicolumn{3}{c|}{Bank Marketing} & \multicolumn{3}{c}{Law School}& \multicolumn{3}{c}{Mfeat} & \multicolumn{3}{c}{COIL} \\
\cmidrule(lr){2-4} \cmidrule(lr){5-7} \cmidrule(lr){8-10} \cmidrule(lr){11-13} \cmidrule(lr){14-16} \cmidrule(lr){17-19}
 & $\mathcal{L}_R$ & $\mathcal{L}_F$ & $\mathcal{L}_C$ & ACC & NMI & BAL & ACC & NMI & BAL & ACC & NMI & BAL & ACC & NMI & BAL & ACC & NMI & BAL \\
\midrule
(A) & \checkmark &  & \checkmark & 0.468 & 0.209 & 0.351 & \textcolor{red}{0.762} & 0.093 & \textcolor{red}{0.302} &\textcolor{red}{0.840} &\textcolor{red}{0.051} &0.407 & \textcolor{red}{0.924} & \textcolor{red}{0.864} & 0.434 &0.739 &0.826 &0.338 \\
(B) &  & \checkmark & \checkmark & 0.460 & 0.201 & \textcolor{red}{0.352} & 0.715 & 0.068 & 0.292 &0.808 &0.049 &0.417 & 0.722 & 0.650 & 0.426&0.445 &0.606 &0.288 \\
(C) & \checkmark & \checkmark &  & 0.461 & 0.207 & 0.350 & 0.690 & 0.085 & 0.289 &0.804 &0.036 &0.417 & 0.832 & 0.828 & \textcolor{red}{0.442}&\textcolor{red}{0.758} &0.839 &0.333 \\
(D) & \checkmark & \checkmark & \checkmark & 0.465 & \textcolor{red}{0.210} & \textcolor{red}{0.352} & 0.749 & \textcolor{red}{0.094} & \textcolor{red}{0.302} &0.836  &0.047 &\textcolor{red}{0.420} & 0.864 & 0.825 & 0.440 &0.750 &\textcolor{red}{0.840} &\textcolor{red}{0.365}\\
\bottomrule
\end{tabular}}
\end{table*}
\begin{figure*}[htbp] 
        
	\centering  
	\vspace{-0.05cm} 
	\subfigcapskip=-5pt 
	\subfigure[ACC]{
		\includegraphics[width=0.25\linewidth]{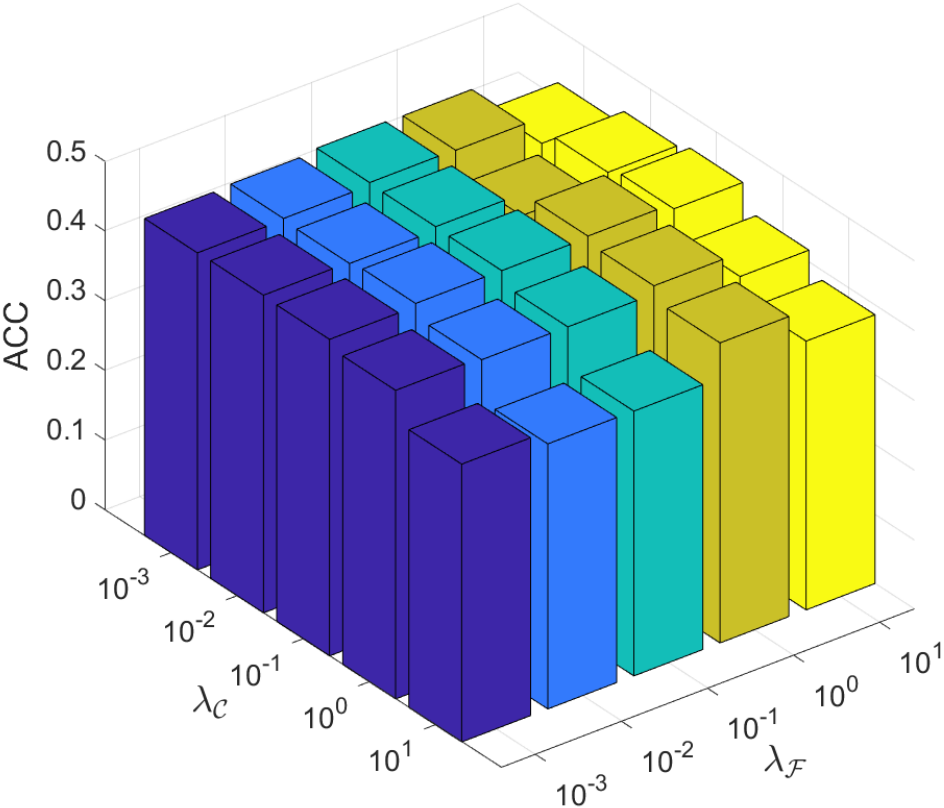}}
        \quad 
	\subfigure[NMI]{
		\includegraphics[width=0.25\linewidth]{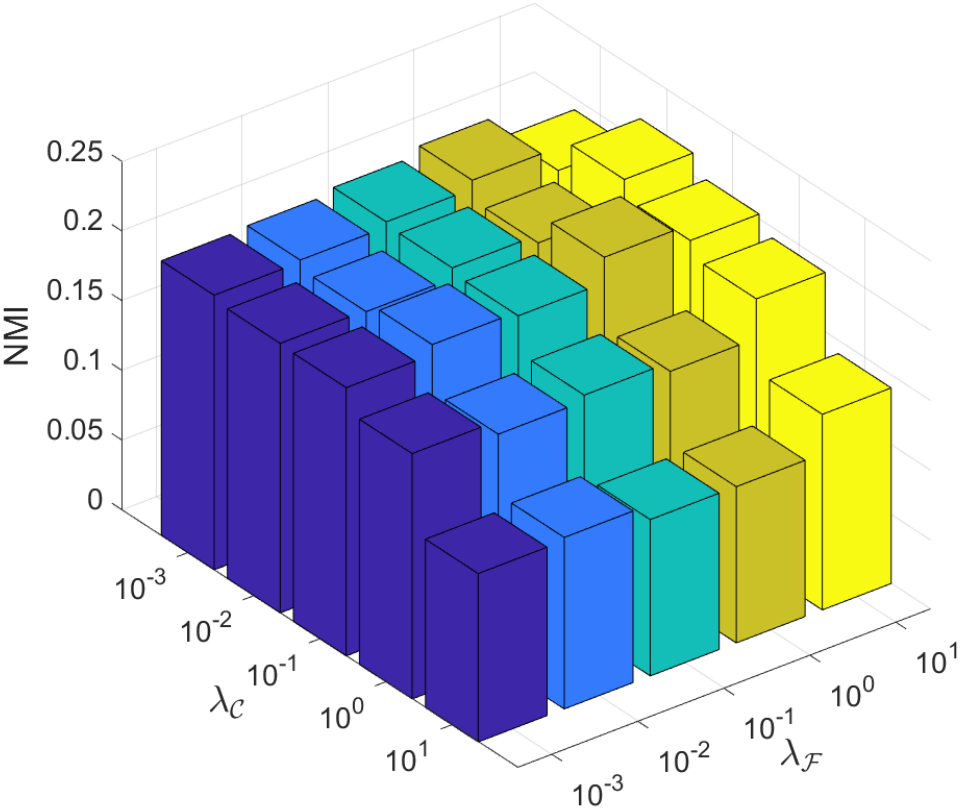}}
	\quad 
	\subfigure[BAL]{
		\includegraphics[width=0.25\linewidth]{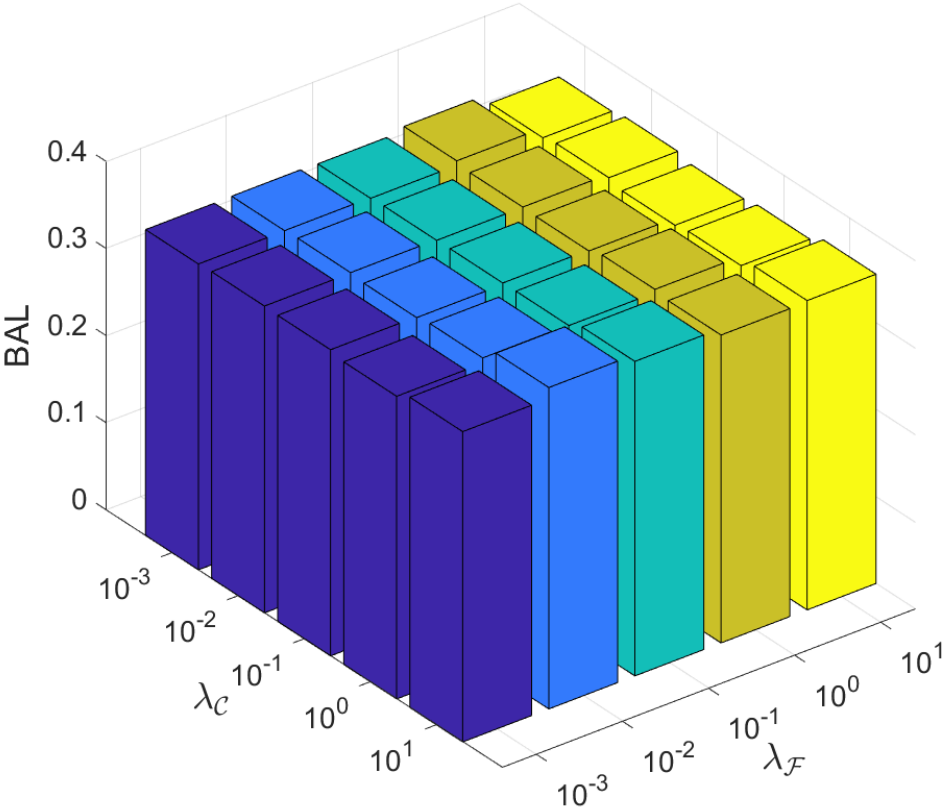}}

	\caption{The effect of parameters $\lambda_C$ and $\lambda_F$  in terms of  ACC, NMI and BAL on Credit Card data set.}
    \label{parameter1}
\end{figure*}

\begin{figure*}[htbp] 
        
	\centering  
	\vspace{-0.05cm} 
	\subfigcapskip=-5pt 
	\subfigure[ACC]{
		\includegraphics[width=0.25\linewidth]{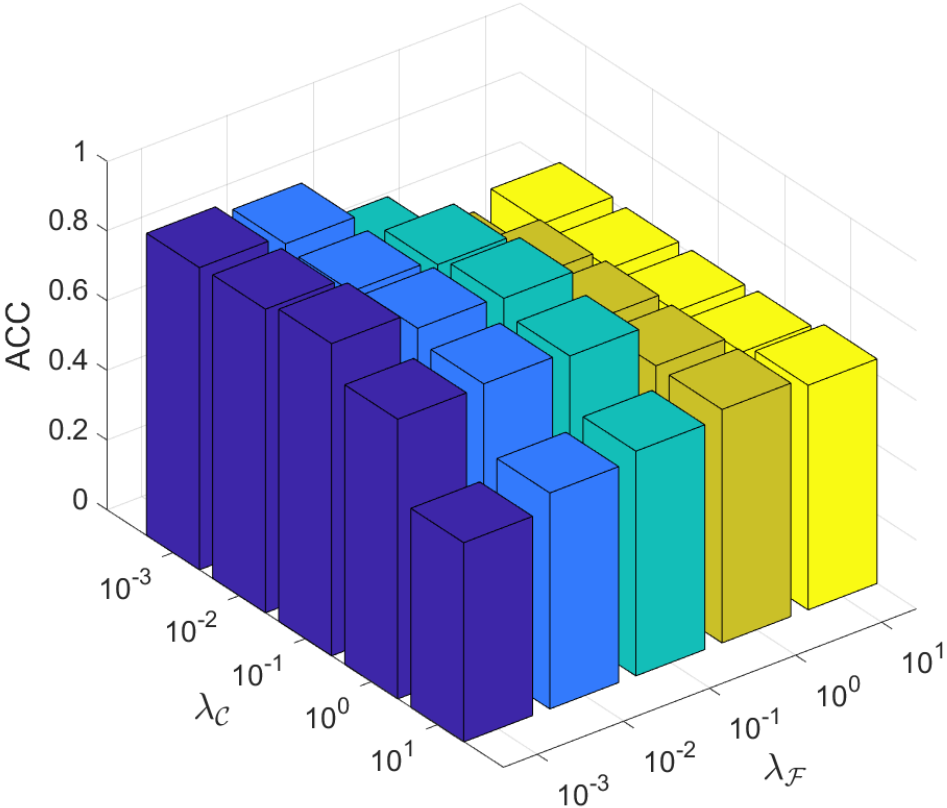}}
        \quad 
	\subfigure[NMI]{
		\includegraphics[width=0.25\linewidth]{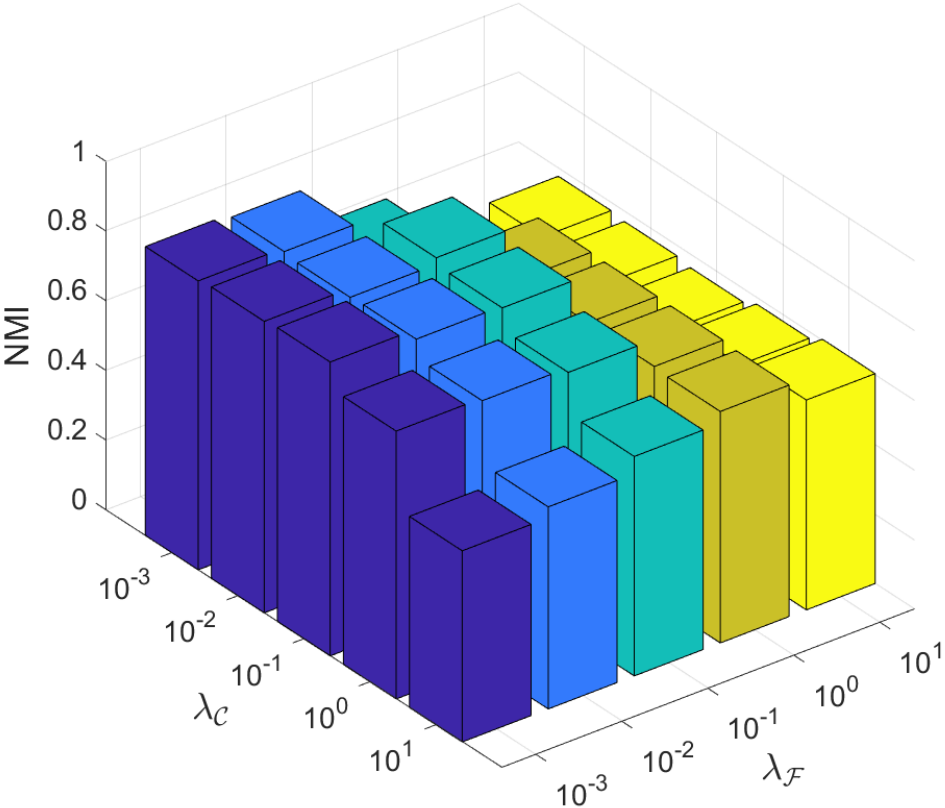}}
	\quad 
	\subfigure[BAL]{
		\includegraphics[width=0.25\linewidth]{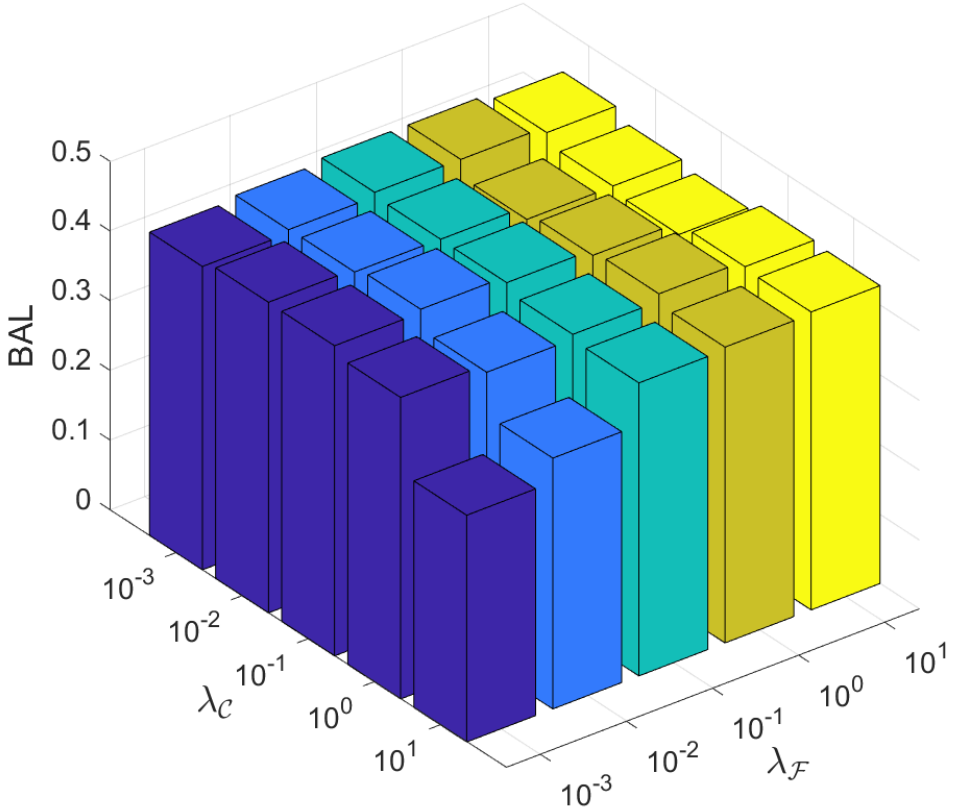}}

	\caption{The effect of parameters $\lambda_C$ and $\lambda_F$  in terms of  ACC, NMI and BAL on Mfeat data set.}

 \label{parameter2}
\end{figure*}
\subsection{Experimental Results}
Table \ref{compare} presents the detailed comparison results, based on which we make the following observations and remarks.
\begin{itemize}
    \item \underline{Overall performance}: AFMVC achieves generally good performance across all data sets and evaluation metrics, despite not attaining the best or the second best results on every individual data set or metric. As shown in Table~\ref{compare}, it achieves the highest average scores in terms of ACC and BAL, and ranks second in NMI. These results demonstrate that the proposed framework is effective in jointly optimizing clustering quality and group fairness through adversarial learning in the multi-view setting.
    \item \underline{Compared with single-view fair clustering method}: AFMVC achieves superior clustering accuracy compared to the three single-view baselines, largely because they rely solely on a single view and are unable to exploit the complementary information available in multi-view data. In particular, AFMVC outperforms VFC on nearly all data sets and metrics. While BFKM achieves competitive BAL scores, its performance in terms of ACC and NMI is still clearly inferior to that of AFMVC. Notably, FFC fails to produce clustering results on the Mfeat and COIL data sets due to excessive memory consumption, as its computation does not scale well to high-dimensional inputs.
    \item \underline{Compared with SOTA multi-view clustering methods}: The overall performance of AFMVC is superior to that of MCPL in terms of both clustering accuracy and fairness. Although its accuracy does not surpass CHOC and CGL in terms of ACC and NMI on standard multi-view data sets such as Mfeat and COIL, it still achieves consistently higher BAL scores. This confirms the effectiveness of our method in improving fairness without significantly compromising clustering quality. 
    \item \underline{Compared with multi-view fair clustering methods}: Since FairMVC is restricted to two-view data sets, it can only be evaluated on Credit Card, Bank Marketing, and Law School. On these data sets, AFMVC generally achieves better performance across all three metrics. In comparison to FMSC, AFMVC delivers higher clustering accuracy while maintaining comparable fairness. Specifically, AFMVC achieves a higher average BAL score, while FMSC ranks better.
\end{itemize}

Moreover, as illustrated in the Fig. \ref{time}, although  AFMVC is not faster than BFKM and FMSC, it demonstrates clear advantages over other methods. Meanwhile, AFMVC exhibits good robustness with respect to both data set size and feature dimensionality, maintaining stable runtime performance under varying conditions.

\subsection{Ablation Study}

To assess the effectiveness of different loss combinations, we conduct an ablation study by selectively enabling or disabling the reconstruction loss $\mathcal{L}_R$, the clustering  loss $\mathcal{L}_C$, and the fairness loss $\mathcal{L}_F$. As shown in Table \ref{ablation}, the full model (D), which incorporates all three losses, consistently achieves the highest BAL scores across all data sets, highlighting the critical role of adversarial fairness learning in promoting fair cluster assignments. Meanwhile, we note that removing the fairness loss may lead to higher accuracy on certain data sets, revealing an inherent trade-off between clustering performance and fairness.

\subsection{Parameter Sensitivity}

AFMVC optimizes a joint objective comprising reconstruction, clustering, and fairness losses. The hyper-parameters  $\lambda_C$ and $\lambda_F$. control the relative importance of these three components, enabling a flexible trade-off among data structure preservation, clustering consistency, and fairness. To assess their impact, we fix all other parameters and conduct a series of experiments to examine how variations in $\lambda_C$ and $\lambda_F$ affect the clustering performance.

Figs.  \ref{parameter1} and \ref{parameter2} present the parameter sensitivity evaluation  results on the Credit Card and Mfeat data sets, respectively (results on other data sets show similar trends). Both $\lambda_C$ and $\lambda_F$  are varied logarithmically from $10^{-3}$ to $10^{1}$. From the figures, we can observe that when $\lambda_C$ is set too large, the reconstruction loss is overly suppressed, potentially distorting the original data representation and leading to a significant drop in clustering performance, as reflected by both ACC and NMI. In contrast, the fairness metric (BAL) remains largely unaffected by variations in $\lambda_C$ and $\lambda_F$, demonstrating the robustness of our method in maintaining fairness.

\section{Conclusion}
\label{5}
In this paper, we propose AFMVC, a fairness-aware multi-view clustering framework that integrates adversarial training to eliminate sensitive information at the representation level. Theoretical analysis further demonstrates that fairness can be preserved through KL alignment of clustering results. Extensive experiments on  data sets with fairness constraints validate the effectiveness of our method in achieving superior fairness while maintaining competitive clustering performance.

For future work, we will focus on improving the robustness of fairness-aware multi-view clustering in more challenging scenarios, such as handling imbalanced sensitive attribute distributions, noisy labels, and outliers. Another promising direction is to extend the proposed method to incomplete multi-view data, where some views are partially missing. These efforts will further enhance the practicality and applicability of our framework in real-world applications.

\appendices

\section*{Acknowledgments}
This work has been supported by the Science and Technology Planning Project of Liaoning Province under Grant No. 2023JH26/10100008, and the  National Natural Science Foundation of China under Grant Nos. 62476038, and 62472064. 

\small
\bibliographystyle{IEEEtran}
\bibliography{refj}

\begin{IEEEbiography}[{\includegraphics[width=1in,height=1.25in,clip,keepaspectratio]{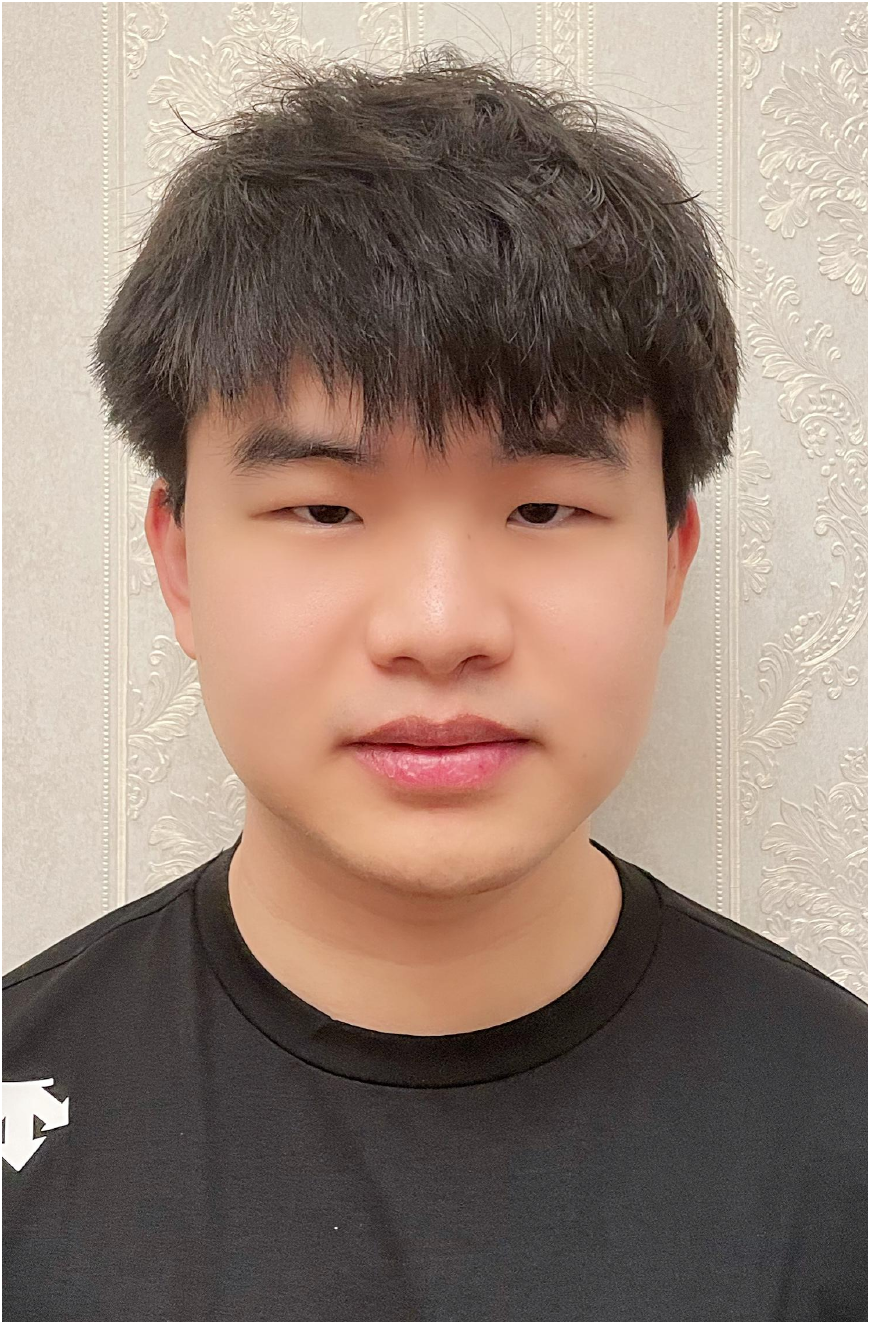}}]{Mudi Jiang}
	received the MS degree in software engineering from Dalian
University of Technology, China, in 2023. He is currently working toward
the PhD degree in the School of Software at the same university. His current
research interests include data mining and its applications.
\end{IEEEbiography}

\begin{IEEEbiography}[{\includegraphics[width=1in,height=1.25in,clip,keepaspectratio]{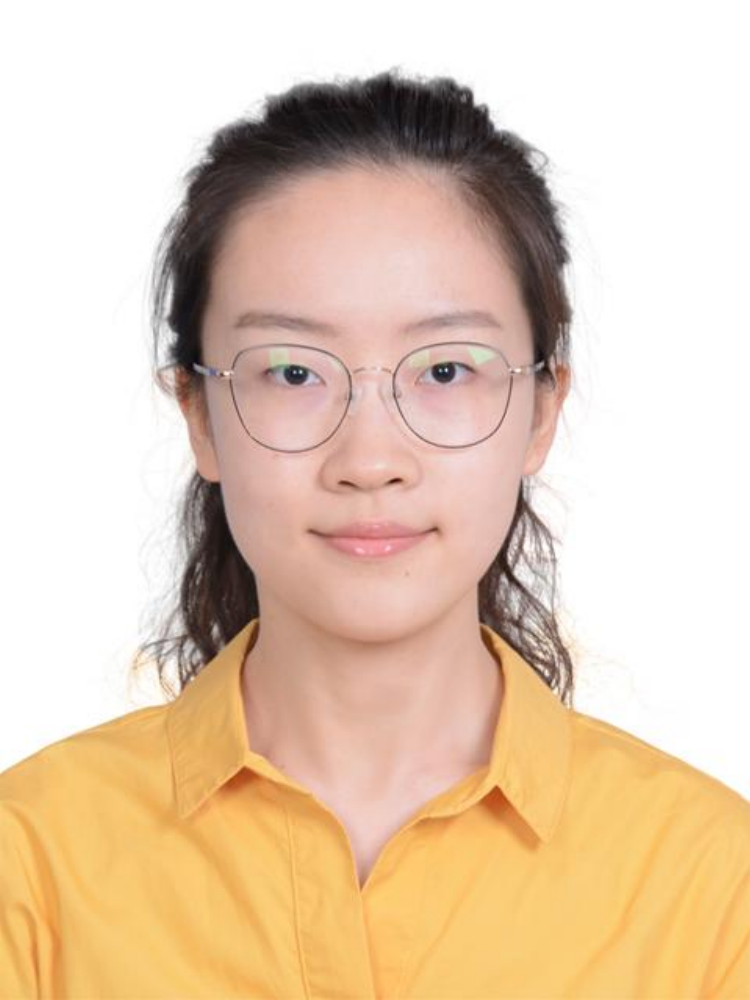}}]{Jiahui Zhou}
	received the BS degree from Dalian University of Technology, China, in 2022. She is currently working toward the MS degree in the School of Software at Dalian University of Technology. Her current
research interests include multimodal retrieval and data mining.
\end{IEEEbiography}

\begin{IEEEbiography}[{\includegraphics[width=1in,height=1.25in,clip,keepaspectratio]{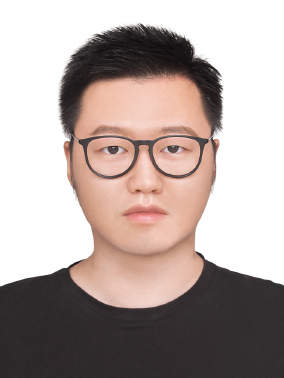}}]{Lianyu Hu}
	received the MS degree in computer science from Ningbo University, China, in 2019.
	He is currently working toward the PhD degree in the School of Software at Dalian University of Technology. His current research interests include machine learning, cluster analysis and data mining.
\end{IEEEbiography}

\begin{IEEEbiography}[{\includegraphics[width=1in,height=1.25in,clip,keepaspectratio]{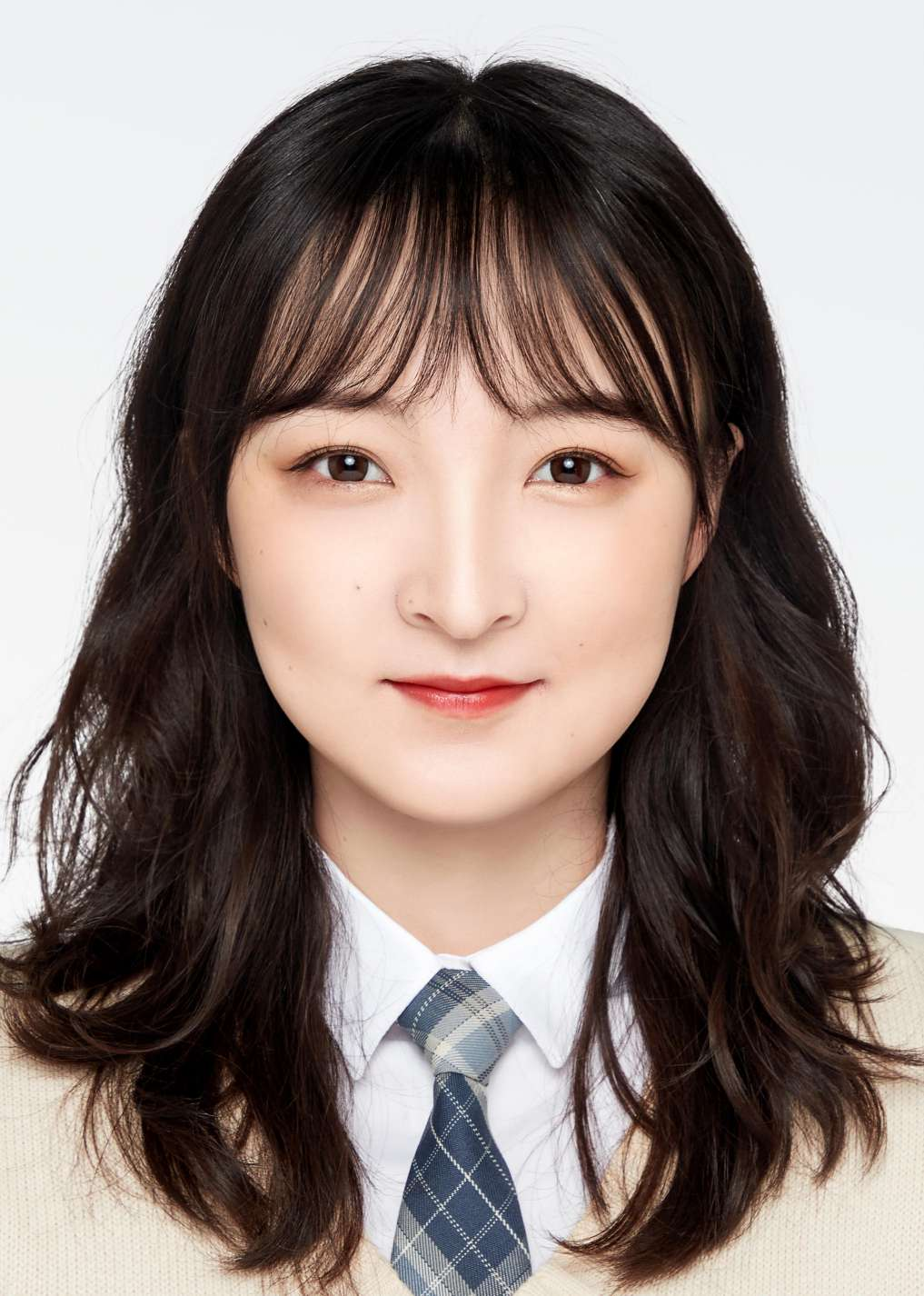}}]{Xinying Liu}
	received the MS degree in Applied Statistics from China University of Geosciences, China, in 2023. She is currently working toward the PhD degree in the School of Software at Dalian University of Technology. Her current research interests include machine learning and data mining.
\end{IEEEbiography}

\begin{IEEEbiography}[{\includegraphics[width=1in,height=1.25in,clip,keepaspectratio]{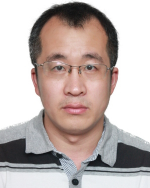}}]{Zengyou He}
	received the BS, MS, and PhD degrees in computer science from Harbin Institute of Technology, China, in 2000, 2002, and 2006, respectively. He was a research associate in the Department of Electronic and Computer Engineering, Hong Kong University of Science and Technology from February 2007 to February 2010. He is currently a professor in the School of software, Dalian University of Technology. His research interest include data mining and bioinformatics.
\end{IEEEbiography}

\begin{IEEEbiography}[{\includegraphics[width=1in,height=1.25in,clip,keepaspectratio]{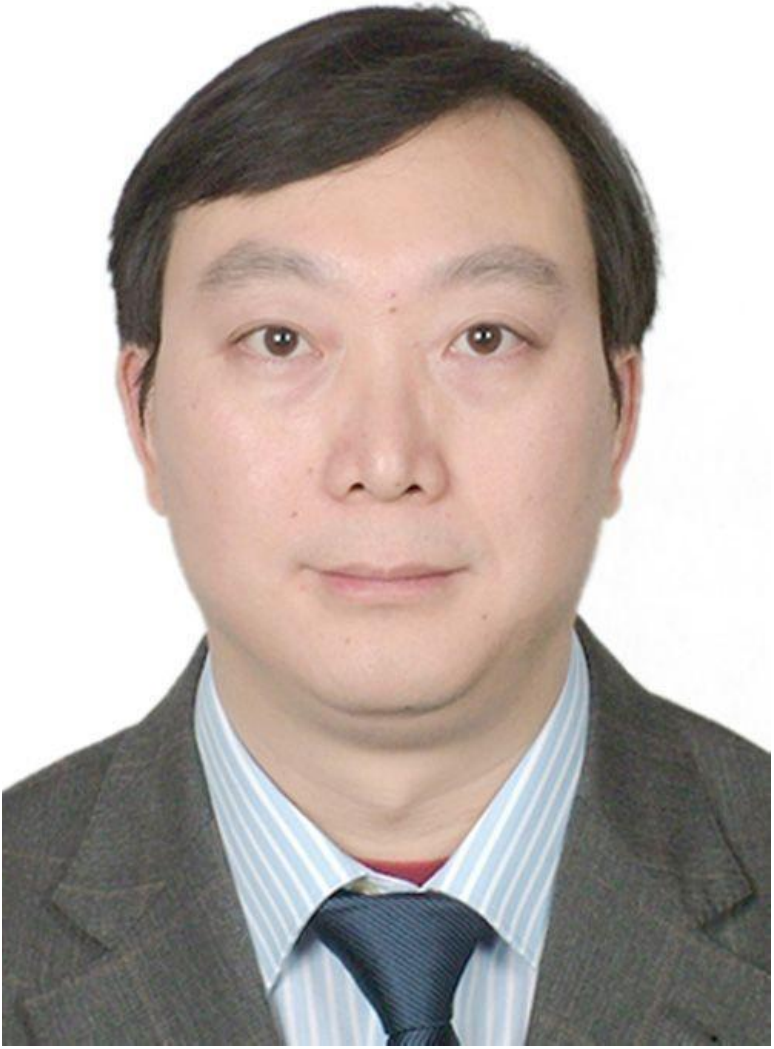}}]{Zhikui Chen}
	(Member, IEEE) received the B.S. degree in mathematics from Chongqing Normal University, Chongqing, China, in 1990, and the M.S. and Ph.D. degrees in mechanics from Chongqing University, Chongqing, in 1993 and 1998, respectively. He is currently a Full Professor with the Dalian University of Technology, Dalian, China.  His research interests are the Internet of Things, big data processing, mobile cloud computing, and ubiquitous networks.

\end{IEEEbiography}

\vfill

\end{document}